\documentclass{article} % For LaTeX2e
\usepackage{hyperref}
\usepackage{url}

\usepackage{graphicx} 

\usepackage{alltt}
\usepackage{algorithm}

\usepackage{amssymb,amsmath,amsthm}
\usepackage{wrapfig}
\usepackage{booktabs}

\usepackage{xcolor}

\usepackage{textcomp}  % Required for encoding \textbigcircle
\usepackage{scalerel}  % Required for emoji \scalerel

\usepackage{charter}
\usepackage{eulervm}
\usepackage[bb=stixtwo]{mathalpha}
\usepackage[scaled=0.9]{beramono}

\usepackage[T1]{fontenc}% NOT OT1
\usepackage{listings}

\makeatletter
\renewcommand*{\verbatim@font}{}
\makeatother

\newcommand{\tab}{\hspace{1em}}

\definecolor{my-green}{HTML}{677d00}
\definecolor{my-light-green}{HTML}{acd373}
\definecolor{my-lighter-green}{HTML}{e6ecce}
\definecolor{my-red}{HTML}{b13e26}
\definecolor{my-light-red}{HTML}{d38473}
\definecolor{my-blue}{HTML}{306693}
\definecolor{my-light-blue}{HTML}{73a7d3}
\definecolor{my-gray}{HTML}{999999}
\definecolor{my-orange}{HTML}{E69500}
\definecolor{my-light-orange}{HTML}{FFC353}

\newcommand{\rcl}[1]{{\color{my-red} #1}}

\newcommand{\h}[0]{\hspace{0.12em}}

\newtheorem{theorem}{Theorem}[section]

\newtheorem{lemma}[theorem]{Lemma}

\def\monkey{\scalerel*{\includegraphics{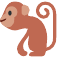}}{\textrm{\Large\textbigcircle}}}
\def\laptop{\scalerel*{\includegraphics{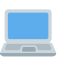}}{\textrm{\Large\textbigcircle}}}
\def\keyboard{\scalerel*{\includegraphics{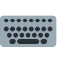}}{\textrm{\Large\textbigcircle}}}

\title{Universal pre-training by iterated random computation}

\author{Peter Bloem \\
Learning \& Reasoning Group, Vrije Universiteit Amsterdam \\
\texttt{up@peterbloem.nl} \\
}

\newcommand{\B}{\mathbb B}
\newcommand{\N}{\mathbb N}

\newcommand{\w}{\mathbold w}
\newcommand{\W}{\mathbold W}
\newcommand{\bE}{\mathbold E}
\newcommand{\M}{\mathbold M}

\newcommand{\bb}{\mathbold b}
\newcommand{\x}{\mathbold x}
\newcommand{\y}{\mathbold y}
\newcommand{\z}{\mathbold z}
\newcommand{\e}{\mathbold e}
\newcommand{\m}{\mathbold m}

\newcommand{\p}{\text{.}}

\newcommand{\is}{$\leftarrow$}

\begin{document}

\maketitle

\begin{abstract}
\noindent We investigate the use of randomly generated data for the sake of pre-training a model. We justify this approach theoretically from the perspective of algorithmic complexity, building on recent research that shows that sequence models can be trained to approximate Solomonoff induction. We derive similar, but complementary theoretical results. We show empirically that synthetically generated data can be used to pre-train a model before the data is seen. We replicate earlier results that models trained this way show zero-shot in-context learning across a variety of datasets, and that this performance improves with scale. We extend earlier results to real-world data, and show that finetuning a model after pre-training offers faster convergence and better generalization. 
 \end{abstract}

\begin{figure}[b!]
		\centering 
        \begin{tabular}{l l}
			{\keyboard\,\monkey} & \lstinline|Z'WY,!X#R_M!IK@JQ!?.>Z\_0&2L%V2G1D4'!| \\
			\hline
			{\laptop\,\monkey} & \lstinline|;5;;'6BUB5CBBBB5Z55BX'X5ZUZZ5P%X555Z5| \\ 
                & \lstinline|E$QFGQ.!XQN*Q,.!.,G**GFFFFF^^FPQ^!YQF| \\
				& \lstinline|\^R5D#**JI,DTTTT,TTTS\TITIDSDT*TTTT\\| \\
%				& \lstinline|AFFFFFFFFFFFF\FFFFFFF!\FFFFFF;FFFFFFF| \\
		\end{tabular}

        \caption{(\protect\keyboard\, \protect\monkey) A string of randomly sampled characters. (\protect\laptop\, \protect\monkey) The result of passing this string through three randomly initialized neural network models. The latter data is partly predictable, and so has value for pre-training.}
        \label{fig:example}
        
\end{figure}
\section{Introduction}
Even in the domain of natural language, where vast amounts of data are available, machine learning is approaching the point where the amount of data on which we would like to train a model is larger than the total amount data available \cite{villalobos2022will}. In many other domains, this point has long been passed. At the same time, as artificial intelligence becomes more and more integrated in production systems, producers of, for instance, literature and visual art, are making it increasingly clear that they do not consent to AI models being trained on their data \cite{blake2025disney,ap2023george}. In short, the question of whether we can do more with less data is increasingly urgent.

One approach is to generate \emph{synthetic data}. This may seem at first to be a false economy: how can we learn something from a data generator that we built ourselves? In information theoretic terms, the data processing inequality \cite[Section~2.8]{cover1999elements} tells us that we cannot increase the information content of a signal by applying computation to it. This seems to preclude the possibility of enriching our models with data from any other source than the real world.

\begin{figure}[t!]

  \centerline{\hspace{-1.2em}
    \includegraphics[width=1.6\textwidth]{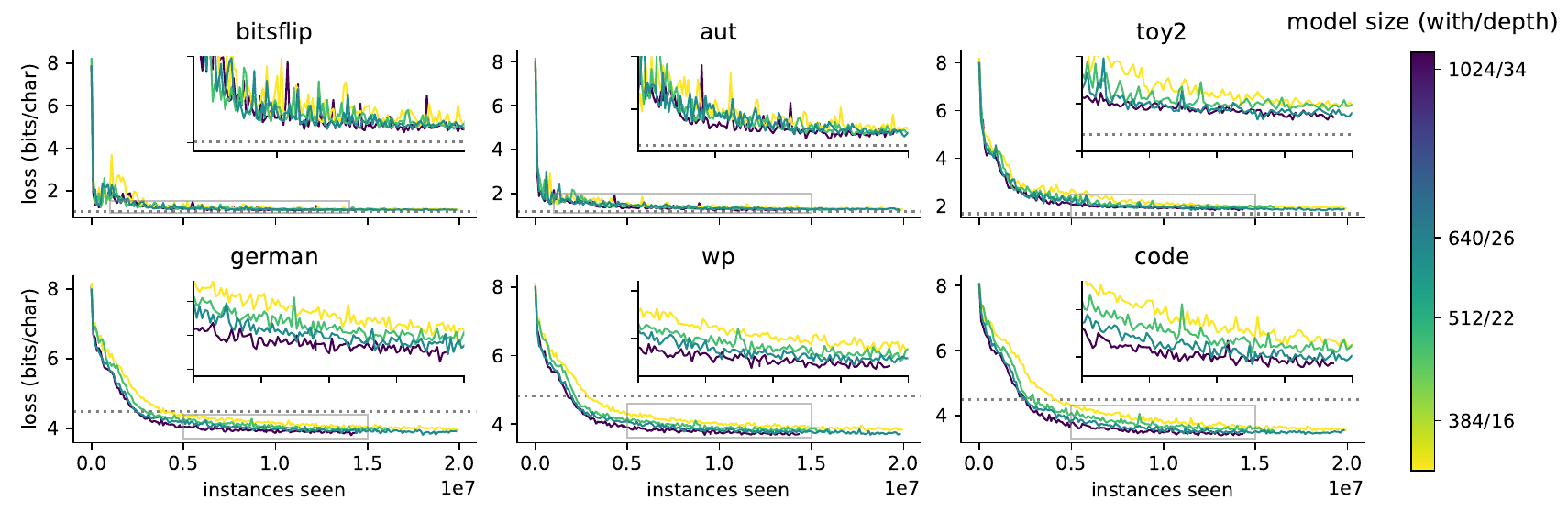}
  }
  \caption{The results of the main experiment (Section~\ref{section:scaling}). We train a transformer model on randomly generated data with computational structure, and test its prediction performance, zero-shot, on six simple datasets every 5000 instances. Horizontal lines indicate the performance of an in-context $n$th-order Markov model (optimized over orders $0$ -- $5$). The results show that the zero-shot behavior (a) is better than chance (8 bits/char) across the board (b) in some cases beats the performance of a Markov model (c) \emph{improves with scale}.}
  \label{figure:scaling}
\end{figure}

However, the key to valuable data is not information content, it's \emph{structure}. The most information-rich signal, a fully random one, is among the least valuable data to learn from. A valuable source of data provides a mix of randomness and structure \cite{bloem2015two}.

Consider the following analogy---from \cite[Section~14.6]{cover1999elements}. A monkey behind a typewriter, bashing keys at random, will famously produce the complete works of Shakespeare in a large, but finite amount of time. A monkey behind a \emph{computer} will also produce the collected works of Shakespeare, and \emph{it will do so more quickly}. This is because there is structure in natural language that may be exploited by a computer program. The second monkey only needs to be lucky enough to type out a computer program that produces the collected works of Shakespeare, whereas the first monkey must be lucky enough to get every single character correct.

The key takeaway is that by taking a source of randomness---whether a monkey or a random number generator---and passing its output through a computer, we can \emph{enrich} data (from the perspective of machine learning). This may well reduce the amount of information in the data (in the Shannon/Kolmogorov sense of the word), but it adds structure. Structure that a machine learning model can learn, and that may transfer to real-world tasks.

With this principle in mind, we ask the following question. Can we pre-train a model before seeing the data? Can we generate rich random data, before knowing what our task is, such that pre-training on this random data benefits learning, regardless of the task? We show that this is indeed possible, and we call the resulting approach \emph{universal pre-training}. We provide some theoretical foundation for this principle, from the field of algorithmic complexity.\footnotemark

\footnotetext{More popularly known as Kolmogorov complexity (or sometimes Solomonoff–Kolmogorov–Chaitin complexity). We will use the phrase \emph{algorithmic complexity} throughout to emphasize its three independent inventors.}

The idea that this can be done is not new. In \cite{grau2024learning}, the authors sample data from an approximation to the universal distribution---the output of a universal Turing machine fed with random bits---and show that a sequential prediction model trained on such data approximates Solomonoff induction, a theoretical, optimal general-purpose learning algorithm.

Separately, in \cite{muller2021transformers,hollmann2025accurate} the authors show that by pre-training on data sampled synthetically from a carefully crafted prior, a transformer model can be trained to perform remarkably strong tabular classification and regression tasks. 

We call this approach---training on synthetic data with the aim of producing a model which predicts well for a broad range of tasks---\emph{universal pre-training} (emphasizing that the method does not originate here, and we merely introduce a single moniker for all these approaches).

Our work makes the following contributions on top of these established methods. 
\begin{enumerate}
\item We establish a theoretical foundation on the basis of class-bounded prefix-free algorithmic complexity\footnotemark~\cite{bloem2014safe}, rather than the monotone algorithmic complexity of Solomonoff induction used in \cite{grau2024learning}. We consider this a complementary, rather than a competing approach, since the important properties are shared between the two frameworks. For probability distributions on finite data, and for arbitrary model classes, the class-bounded, prefix-free formalism may offer some benefits.
\item We modify the approach first demonstrated in \cite{grau2024learning} of sampling data from some approximation to the universal distribution. We specifically focus on the idea of enriching data by passing it through random computations, and show theoretically that this leads to (a) hierarchies of model classes, and (b) an approximation to the universal distribution in the limit. 
\item We evaluate from the perspective of pre-training. We scale up compared to the models used in \cite{grau2024learning} and demonstrate that the method works on real world data including natural language and code. We observe that performance improves with scale and we show benefits to finetuning after universal pre-training. 
\item We provide synthetic data generators for evaluation. For some of these, our model does not yet outperform a simple in-context Markov model, while by design a better solution does exist, providing a sound challenge for future work in this direction.
\end{enumerate}

\noindent As a whole, we hope that our perspective shows the importance of \emph{computational depth} \cite{antunes2006computational}. That is, random data can be made more valuable for pre-training, but the amount of value is related to the amount of \emph{computation}. As such, universal pre-training is a data-computation tradeoff. We can get by with less data by spending more on computation.

Figure~\ref{fig:example} shows a simple example of our perspective: a string of ASCII characters, each chosen uniformly at random, and the result of passing that string through a randomly initialized neural network model. The first string is fully random, and nothing useful may be learned from it. The second string contains structure: certain characters and structures re-occur in a predictable way. We can make sensible predictions about one part of the string by studying another part of the string. 

Some of these basic patterns may transfer to other domains. For instance, the relative frequency of a character at the start of the string, is more likely than not be a reasonable prediction of its probability at the end of the string, and this pattern will hold across a wide variety of tasks.

 Other patterns, such as the specific characters that are frequent, are entirely specific to the current sample, and will disappear if we sample a second string in the same way.

Ultimately, if universal pre-training works well, patterns that are likely to be \emph{universal}---that is, shared by many datasets generated from computational sources---will be learned by the model, whereas patterns that are specific only to some sources, will quickly be averaged out. 

The idea of universal patterns may seem to fly in the face of the no-free-lunch theorem \cite{wolpert1997no}, which shows that no model outperforms any other when averaged over all tasks. The idea put forward in \cite{grau2024learning}, and which we aim to build on here is that while no patterns are universal over all datasets, some patterns are  universal over those datasets that were generated by the combination of a computational process and a source of randomness. This assumption is very broad and likely captures any data which we may hope to be able to analyse by computational means. Under these assumptions, as has been established before \cite{lattimore2013no,goldblum2023no}, the no-free-lunch theorem no longer applies, and universal patterns can be shown to exist.

We hope to emphasize here that this is not just a theoretical curiosity, but that universal pretraining shows  practical benefits already. First, the better we can approach universal pretraining, the more we may amortize the training investments over future tasks. In the limit, we end up with a model that can be trained once at great expense and then be used to perform learning on all tasks by in-context learning or finetuning. In addition to good in-domain performance at low amortized cost, such a model also promises exceptional out-of-domain generalization, since the base model was trained for universal induction.

In the conclusion, we suggest a series of key research directions that should be pursued to establish whether universal pre-training may become a viable approach at scale in the future. We also pause to consider what the broader impact may be, if this data/compute tradeoff is chased blindly, based on optimistic predictions, and provide some caution.

All code is available at \url{https://github.com/pbloem/up}.

\section{Preliminaries}

We will briefly review the required preliminaries. See \cite{ming1990kolmogorov} for a more detailed treatment. These are mostly required to understand the theoretical framework we use to show that universal pre-training is feasible. To understand the experimental part of the paper, the intuitive explanation in the introduction should suffice.

Let $\B$ be the set of finite binary strings $\{\epsilon, 0, 1, 00, 01, 10, 11, 000, \ldots\}$ with $\epsilon$ the empty string. Let $|x|$ be the length of string $x \in \B$ in bits. We write the concatenation of bitstrings $a$ and $b$ as $ab$.

Let $p(x)$ be any probability distribution on any countable set: then there exists a prefix-free code \cite[Section~5.1]{cover1999elements} on that same set such that the length $L(x)$ of the codeword for $x$ follows the relation $L(x) = - \log_2 (x)$ up to one bit. We will hand-wave this minor difference and equate codelengths with negative log-probability as explained in \cite[Section~3.2.2]{grunwald2007minimum}.

We will use Turing machines \cite{turing1936computable} to model computable probability distributions on $\B$. To this end we fix Turing machines with an \emph{input tape}, a \emph{conditional} tape and an \emph{output} tape, as well as some number of work tapes.

The read-head on the input tape is only allowed to move from left to right: the effect of this constraint is that after reading a bit from the input tape, the machine either halts, reads the next bit, or computes infinitely without doing either. This means that the set of inputs on which a machine halts forms a prefix-free set: no halting input is the prefix of another halting input. This is called a prefix-free Turing machine \cite[Chapter~4]{ming1990kolmogorov}. Note that the heads on the conditional, output and work tapes can move in both directions.

If a machine $T$ halts, we take the input bits it has read as its input $x$, $z$ as an optional conditional value, and the value on its output tape when it halts as its output $y$, writing $T(x, z) = y$. We write $T(x)=\infty$ for an input on which the machine does not halt (that is, it processes indefinitely without advancing the read-head). We write $T(x) := T(x, \epsilon)$.

We can now use such machines to simulate probability distributions, by feeding them random bits. We start the machine, and whenever it reads from the input tape, we sample a random bit and provide it. If the machine halts, we take the output as our sample.\footnotemark

\footnotetext{Since the machine may not halt, we may end up with a deficient distribution on $\B$. However, we will mostly deal with subsets of the Turing machines that always halt, if given sufficient input. If we place no restrictions on the Turing machines used, we have defined the class of \emph{lower semicomputable semimeasures} \cite[Lemma~1]{bloem2014safe}. We will refer to these as the computable distributions for the sake of simplicity.}

We assume that we are given some enumeration of all such Turing machines $T_i$, with $i\in 0, 1, 2, \ldots$. We refer to the probability of sampling $y$ from Turing machine $i$ as $p_i(y)$. Since each input $x$ is sampled with probability $2^{-|x|}$, we can write $p_i$ as

$$
p_i(y) = \sum_{x : T_i(x) = y} 2^{-|x|} \p  
$$

Due to the universality of Turing machines, any distribution from which we can sample by a computational process fed with random bits is simulated by a Turing machine in this enumeration. This includes for example, a sample from a normal distribution, from a Markov model or from a probabilistic context-free grammar. Moreover, every probability distribution whose probability function is computable (or lower semicomputable) is represented in the enumeration by a Turing machine that samples from it.

Next, assume a pairing function $(i, x)$ that encodes $i$ and $x$ into a single string in $\B$. We can define a universal Turing machine $U((i, x)) = T_i(x)$, which unpacks the string $(i, x)$ (assuming some binary encoding of the natural numbers $i$) and then simulates the running of the $i$-th Turing machine on input $x$. If the pairing function uses prefix-free codewords, which is easy to achieve, then $U$ is itself a prefix-free Turing machine, which means that it occurs somewhere in the enumeration $T_i$. \footnotemark

\footnotetext{A simple way to achieve a prefix-free pairing function is to define a prefix-free encoding of $x \in \B$, written as $\overline x$. We can then use the pairing function $(i, x) = \overline \imath\, \overline x$, since concatenating two prefix-free codewords results in another prefix-free code.}

%A simple way to achieve a prefix-free pairing function is to define a prefix-free encoding of $\B$, written as $\bar x$. We can then use the pairing function $(i, x) = \bar i \bar x$, since concatenating two prefix-free codewords results in another prefix-free code.}

With this, we can sample from $U$, feeding it random bits until it halts. This effectively samples a Turing machine $i$ according to some prior probability $p(i) = 2^{-|i|}$, and then samples an output $y$ according to  $p_i(y)$. We call this the \emph{universal distribution} $m(y)$. The universal distribution is a Bayesian mixture over the class of computable distributions.

To emphasize the contrast in approach with \cite{grau2024learning} and Solomonoff induction in general \cite{solomonoff1964formal1,solomonoff1964formal2}, note that in that framework the Turing machines are \emph{monotone} rather than \emph{prefix-free}. Specifically, that means that the output tape is also one-way. This is useful in modeling infinite sequences sequentially, while prefix-free Turing machines are better suited to modeling finite objects. We show below that the properties that are important for current purposes are shared by both frameworks.\footnotemark

\footnotetext{This result is not novel. We merely extract the relevant parts of the known theory \cite[Chapters 4 and 5]{ming1990kolmogorov} and restate them in our notation. The equivalence of the prefix-free framework and the monotone framework follows from the relation between the universal discrete semimeasure $\m(x)$---which corresponds to $p_U$---and the universal continuous semimeasure $\M(x)$---which emerges from the monotone framework.}

\section{Related work}

\paragraph{Synthetic data}

Training on synthetic data goes back at least to the simple principles of data augmentation \cite{wang2024comprehensive}. In recent years, much research has been devoted to the use of domain-specific simulators generating synthetic data to encode very specific domain assumptions. For instance, in training self-driving cars, robotics control, tabular data and so on \cite{shafaei2016play,collins2021review,hollmann2025accurate}.

When pre-training on purely synthetic data, the generator of this data is a probability distribution, which functions much like a Bayesian prior. This prior is usually easy to sample from, but hard to perform inference with. One perspective on pre-training then becomes that we are fitting a model to data sampled from our prior so that we have a prior that allows for some form of inference. This view is called a prior-data fitted network (PFN) \cite{muller2021transformers}. The models in this paper (as well as those in \cite{grau2024learning}) fit this category. 

The argument is often made that PFNs require carefully tuned priors that in some sense overlap with the expected test domain \cite{nagler2023statistical,muller2021transformers}. \cite{grau2024learning} offers a more subtle view: if the source of pre-training data is the universal distribution, then in the limit of long sequences the PFN predicts as well in context as any computational distribution (including the true source of the data). We cannot, however, sample from the universal distribution, so apprixmations are needed. Nevertheless, we can still make the required claims of universality if the pre-training distribution dominates the source of the data, as explained in the next section.

Both in synthetic and non-synthetic pre-training domains, a trend is emerging suggesting that domain specific features are not required: all that is often needed is a large variety of structural features. For example, in \cite{nakamura2024scaling} a simple dataset of fractals was found to be beneficial for pretraining a visual model, while in \cite{venkataramanan2023imagenet}, the frames of a single, long video were shown to be as valuable in pre-training as the whole of ImageNet. 

The perspective of universal pre-training offers a perspective on this trend: it is not world knowledge that is primarily learned from such datasets, but general computational structure. In the limit, we can train on all possible computational structure and apply any relevant learned patterns during inference.

%The limit of this trend is universal pre-training: making only a single assumption on the source of the data, namely that it is computational in nature, results naturally in samples from the universal distribution.

\paragraph{The No-Free Lunch Theorem}

The idea of pre-training a model irrespective of its task seems to fly in the face of the no-free-lunch (NFL) theorem \cite{wolpert1997no}, which states that, averaged over all tasks, all models perform the same. This suggests that if our pre-training makes any change at all to the model's performance on a subset of tasks, it must reduce its performance on other tasks.

This is indeed true, but it's possible to get out of the NFL theorem with a very broad, almost universal assumption about our data \cite{lattimore2013no}: that the process $p$ that generated it is \emph{computational}. That is, it can be simulated by a probabilistic Turing machine. Under this assumption, the source of our data is \emph{dominated} by the universal distribution $m$, and by any class-universal distribution $m_C$ for which $p \in C$. 

This means that there are computable probability distributions $m_C$ that fit our data as well as its source $p$ does (up to a multiplicative constant). Moreover, this holds for all $p \in C$. Therefore, $m_C$ provides a universal predictor for $C$. This does not violate the NFL theorem, because we do make an assumption about the data, namely that it comes from a source in $C$. However, $C$ can be made \emph{very} broad---for instance the class of all polynomial-time Turing machines---while still allowing practical sampling \cite{bloem2014safe}.

\section{In theory}
\label{section:theory}

We will first build on the preliminaries to establish some useful theoretical properties. In the next section, we will use these to inform a practical approach to universal pre-training in the context of sequential transformer models.

Let $C$ be a set of computable functions from $\B$ to $\B$. If we define some prior $p$ on $C$, we can sample a random model from $C$, feed it random bits $r$, and observe the output $x$. The probability of sampling some specific $x$ from this process is 
\[
m_C(x) = \sum_{c\in C} p(c)\;p_C(x) = \sum_{c\in C,r \in R} p(c)\;2^{-|r|} \;\;\;\text{with}\;\;\; R = \{r \mid c(r) = x\}\p 
\]

%
%By the logic spelled out in the introduction, the string $x$ is more likely to contain interesting structure than the original random bits $r$. The central idea of this paper is to sample data $x$ and pre-train a sequential model on such samples.

We will build our theoretical argument in three steps. First, we will show that sampling $c$ and feeding it random bits $r$ until it halts corresponds to sampling from a class-universal distribution for $C$. The code corresponding to $m_C$ with codelength function $L_C(x) = - \log m_C(x)$  compresses $x$ at least as well as any member of $C$, up to an additive constant. This is called \emph{domination}, as in $m_C(x)$ dominates any member of $C$.

%
%\footnotetext{Distribution $q$ dominates distribution $p$ if $\forall_x -\log q(x) \overset{+}{>} -\log p(x)$ or equivalently $\forall_x q(x) \overset{\times}{<} p(x)$. The former definition is likely more intuitive, but we use the latter form in most of the paper.}

Second, under reasonable assumptions, passing the generated data through another member of $C$---iterating the process---extends the model class from which we are sampling, thus creating a hierarchy of classes.

%That is, every time we iterate this process we are extending the class of distributions we dominate. This creates a hierarchy of model classes.

Third, we show that while the limit of this hierarchy is not guaranteed to be the universal class---that is, the set of all computable distributions---it is possible to set up our model classes so that this is is the limiting case. Specifically, we show that this is possible with a model class of (recurrent) neural networks with randomly sampled weights

This suggests that sampling random bits and feeding them iteratively through a recurrent neural network will approximate the universal distribution in the limit of (a) the number of iterations, and (b) the allowed size of input/output sequences. This holds even if the neural networks themselves are bounded in size.

\subsection{Model-bounded algorithmic complexity}

Let $C$ be an enumerable subset of Turing machines. Let $p(c)$ be a computable distribution on the members $c$ of $C$. 

Let $m_C(x)$ be distribution from which we sample by first sampling $c$ from $p(c)$ and then sampling $x$ from $c$ as described in the preliminaries.

In \cite{bloem2014safe}, it is shown that for all $x$ and all $c$ in $C$:

\[
m_C(x) \overset{\times}{>} p_c(x) \hspace{2em}\text{or equivalently}\hspace{2em}
-\log m_C(x) \overset{+}{<} - \log p_c(x)
\]

Where $a(x) \overset{+}{<} b(x)$ indicates that $a(x)$ is bounded from above by the function $b(x) + c$ for some constant $c$ and similarly, $a(x) \overset{+}{>} b(x)$ means that $a(x)$ is bounded from below for some constant $c$.

This result shows that $m_C$ always assigns any $x$ greater probability than any individual member of $C$, up to a multiplicative constant independent of $x$ and equivalently the code corresponding to $m_C(x)$ compresses better than the code corresponding to any individual member of $C$ up to an additive constant number of bits. 

The relevance to our current purpose is that if we sample data from any distribution $c \in C$, and then use $m_C$ to compress it, our regret (for not having used the optimal $c$) will be bounded by a constant.

In \cite{bloem2014safe} conditions are given under which $m_C()$ can be approximated computably: all TMs in $C$ must be \emph{sufficient} (they must halt eventually if fed sufficient bits). This also ensures that $m_C$ can be sampled from in finite time.

The following lemma shows a simple domination result between classes. 

\begin{lemma}
	For model classes $C, D$ if $D$ contains a turing machine $u((i, x)) = T_i(x)$ with $i$ enumerating $C$ and $(,)$ a prefix-free pairing function, then $m_D$ dominates $C$.
\end{lemma}
\begin{proof} 
\begin{align*}
m_D(x) &= \sum_{d \in D, r \in R} p(d) 2^{-|r|} \text{ with } R(r \mid d(r) = x)\\
& \geq \sum_{r} p(u) 2^{-|r|} = p(u)m_C(x) 
\end{align*}
\end{proof}

This type of result allows us to build a hierarchy over model classes, from very simple computable classes at the bottom to increasingly complex classes, requiring larger and larger amounts of computation toward the top. 

The top of the hierarchy is the \emph{universal distribution} $m$ which we get for the class of all (prefix-free) Turing machines. It is well known that $- \log m(x)$ is equal (up to a constant) to the prefix-free algorithmic complexity $K(x)$. 

The key proposition behind universal pre-training is that we sample from $m_C$, with $C$ as high up in the hierarchy as we can go,\footnotemark~and train a model on such data to approximates $m_C$. Such a model (up to the approximation quality) would dominate any model in $C$. If the source of our data is in $C$, predicting with $m_C$ gives us bounded regret, without ever seeing the data. 

\footnotetext{It may be tempting to think that we can sample from the universal distribution, since we can simply feed the UTM random bits until it produces an output. However, since our patience will always be bounded (by our lifetime or the lifetime of the universe), we would actually be sampling from a time-bounded UTM (which corresponds to a class $C$ somewhere in our hierarchy).}

\subsection{Sequential prediction}

In Solomonoff induction---and by extension in \cite{grau2024learning}---the theoretical framework is one of monotone Turing machines producing infinite strings, applied to sequential prediction: predicting the bits in a string one-by-one, each conditioned on the last. 

In our framework, since we use prefix-free Turing machines, we only model finite strings, and we primarily describe the probability of the whole string rather than its individual bits.

In this section, we will show that the frameworks are fundamentally compatible in the important properties.

First, for a probability $p$ in $\B$ we write the conditional probability of seeing a \emph{prefix} $x$ continue with the bit $b$ as $p(b \mid x)$. This is defined as

\[
p(b \mid x) = \frac{p(xb\_)}{p(x\_)}
\]
where $x\_$ is the set of all finite strings with the prefix $x$ (which includes $x$ itself). That is, if we have sampled a string from $p$ and we read it from left to right, once we have read the bits $x$, the probability that the next bit is $b$ is given by $p(b \mid x)$.

It is not given that if $q$ dominates $C$, that $q(b\mid x)$ always dominates $p_c(b\mid x)$. In practice, $p_c(b \mid x)$ may be much larger for individual strings than $q(b \mid x)$. However, we can show that in aggregate, these fluctuations disappear quickly as strings grow longer.

Let 
\begin{align*}
D_n &= \sum_{|x| = n} p(x\_) KL( p(b \mid x), q(b \mid x) ) \\
&= - \sum_{|x| = n} p(x\_) \sum_b p(b \mid x) \log \frac{p(b \mid x) }{q(b \mid x)} \p \\ 
\end{align*}
That is, if we sample a prefix $x$ of length $n$ from $p_c(x)$, $D_n$ is the expected KL divergence between the distribution that $p_c$ assigns to the next bit, and the distribution that $q$ assigns to the next bit. 

We can now show the following.

\begin{theorem}[Adapted from Theorem~5.2.1 in \cite{ming1990kolmogorov}]
If $q$ dominates $p$, then $\sum_{n=1}^{\infty} D_n$ is bounded. 
\end{theorem}
\begin{proof}
First, we take the sum up to some value $k$. 
\begin{align*}
\sum_{n=1}^k D_n &= - \sum_{n, |x| = n, b} p(x\_) p(b \mid x) \log \frac{p_c(b \mid x) }{q(b \mid x)} \\
&=  - \sum_{n, |x| = n, b} p_c(x\_) \frac{p(xb\_)}{p(x\_)} \log \frac{p_c(b \mid x) }{q(b \mid x)} \\
&=  - \sum_{n, |x| = n, b} p_c(xb\_) \log \frac{p_c(b \mid x) }{q(b \mid x)} \\
\end{align*}	

Now note that if we extend $xb$ with an arbitary suffix $y$ of length $m$, we get 

\begin{align*}
\sum_n D_n &= \sum_n\sum_{|x| = n, b, |y| = m} p_c(xby\_) \log \frac{p_c(b \mid x) }{q(b \mid x)}
\end{align*}

This is because $p_c(xb\_) = \sum_{|y|=m} p_c(xby\_)$ and the quantity inside the logarithm is the same for all $xby$. Choose the length of $y$ so that $n + 1 + |y| = k$. Then  

\begin{align*}
\sum_{n=1}^{k} D_n &= \sum_n \sum_{|z| = k} p(z\_) \log \frac{p(z_{n+1} \mid z_{1:n}) }{q(z_{n+1}\mid z_{1:n})} \\
&= \sum_{|z| = k} p(z\_) \sum_n  \log \frac{p(z_{n+1} \mid z_{1:n}) }{q(z_{n+1}\mid z_{1:n})} \\
&= \sum_{|z| = k} p(z\_)  \log \prod_n \frac{p(z_{n+1} \mid z_{1:n}) }{q(z_{n+1}\mid z_{1:n})} \\
&= \sum_{|z| = k} p(z\_)  \log \frac{p(z\_) }{q(z\_)}
\end{align*}

Since $q$ dominates $p_c$, the fraction in the logarithm is bounded by a constant independent of $z$ and $n$, so that 
\begin{align*}
\sum_n D_n &\overset{\times}{<} \sum_{|z| = k} p_c(z\_) \leq 1
\end{align*}
\end{proof}

%In \cite[Theorem 5.2.1]{ming1990kolmogorov} it is shown that this additionally implies that the squared distance between conditional probabilities by $p_c$ and $q$, summed over all lengths $n$, is also bounded. 

As a consequence, since $m_C$ dominates all models in $C$, it has a bounded cumulative expected KL divergence $\sum_n D_n$ against all models in $C$. If we don't know what the source of our data is, but we do know that it comes from $C$, then using $m_C$ as our predictor is optimal in the limit. 

This shows first that the primary property of Solomonoff induction holds for prefix-free Turing machines as well as for monotone Turing machines: if we perform sequential prediction, then in the limit of large data, the predictions of a dominating model $q$ converge to be as good as any dominated model $p$. 

Second, it shows that if we limit ourselves to a subclass $C$ of Turing machines for reasons of practicality, we still have this property, so long as our data is produced by a member of $C$. If we set $C$ to the class of all Turing machines, we recover the universal setting of the original Solomonoff induction (albeit on a basis of prefix-free TMs). 

The benefit of our perspective is that we can now talk about any class of probability distributions (or semi-measures) on finite objects. So long as these are (semi)-computable, we can cast them as a subset of the Turing machines\footnotemark and fit them in our framework. We believe this provides a simple and natural bridge between probability distributions used in daily practice and the language of algorithmic complexity. 

\footnotetext{Note that we don't need to identify \emph{all} Turing machines that compute our class members---this is impossible by Rice's theorem. We only need to find one Turing machine for every member of our class.}

\subsection{Iterative sampling}

The example from the introduction provides a simple intuition: taking random noise and adding random computation to it results in data that is more valuable for pre-training than the original noise. 

This immediately suggests a natural followup question: if we take this "enriched noise" and add even further computation to it, does it become, in some way, enriched further? We show that in carefully controlled circumstances, this is indeed the case.

To ensure that any amount of random bits can be used, we will model this process by placing random bits on the \emph{conditional} tape (placing them on the input tape would require a specific amount of bits for the Turing machine to produce an output).

Let $p_u$ be a base distribution on $\B$ which is uniform in the sense that all strings of the same length have equal probability.

Let $m^1_C$ be the distribution defined by the process of sampling $u$ from $p_u$ and $c$ from $p(c)$ and computing $c(r, u)$, where $r$ are the random bits fed to the Turing machine's input tape and $u$ is placed on the output tape. That is, we are sampling from $p_C(x \mid u)$. Call $m_C^0 := p_u$. Let $i$ be the Turing machine which always copies the conditional input to the output, ignoring the random bits on the input tape.

Let $m^{n+1}_C$ be the distribution defined by the process of sampling $u$ from $m^n_C$ and $c$ from $p(c)$ and computing $c(r, u)$. 

\begin{theorem}
Let $u, i \in C$. Then $m_C^{n+1}$ dominates $m_C^n$.	
\end{theorem}
\begin{proof}

\begin{align*}
m_C^{n+1}(x) &= \sum_{u \in \B, c \in C} m_C^n(u) \, p(c) \, p_c(x \mid u) \\
&= \sum_{u, c} \left(\sum_{u', c'} m_C^{n-1}(u')p(c') p_{c'}(u \mid u')\right) p(c) \, p_c(x \mid u) \\
&\geq \sum_{u, c} \left(\sum_{u'} m_C^{n-1}(u')p(i)p_i(u\mid u')\right) p(c)\, p_c(x \mid u ) \\
&= \sum_{u, c} m_C^{n-1}(u) \, p(i) \, p(c) \, p_c(x \mid u) \\
&= p(i) \, m_C^n(x)
\end{align*}
\end{proof}

The conclusion is that by sampling random noise, feeding it through a random Turing machine and iterating this process, we are generating increasingly rich data (in the sense that we are getting closer in our hierarchy to the universal distribution). 

%This does not represent any kind of free lunch, since we are paying extra computation for the extra richness, but we will see in the next section that ther

%we can get a benefit by generating data in this fashion, and training on all the intermediate sequences. That is, every sequence we train on costs us one pass through a randomly chosen model, but we are training, for some small percentage of our data, on data sampled from $m_C^n$ with arbitrarily high $n$.

A few points are worth making explicitly. 
\begin{enumerate}
\item The end result of this iteration is not guaranteed to be the universal distribution. It may simply converge to some point where $m_C^n(x) \overset{\times} = m_C^{n-1}(x)$ for all $n$ above some value. However, we show in the next subsection that it is possible to set up a realistic model class (of recurrent neural networks) in such a way that we do approximate the universal distribution. 
\item Applying the theorem $k$ times, we see that  $m^n_C(x) \geq p(i)^km_C^{n-k}$.
This shows that with every step of the iteration the constant by which we dominate shrinks exponentially. The effect is easiest to interpret in log-space: $- \log m^n_C(x) \leq  - k \log p(i) -\log m_C^{n-k}$. This shows that for every iteration we add to the number of bits of slack we need to accept before we can say that $m_C^n$ compresses as well as $m_C^{n-k}$.
\item If the iteration converges to the universal distribution $m(x)$ in the limit, we know that this growth in the constant can be bounded, since the universal distribution is known to dominate any $m_C^n$ with a constant bound. That is, the argument above yields a bound which diverges with $k$, but a constant bound also exists. However, this only applies if the iteration converges to a distribution that samples from $m(x)$ with non-vanishing probability (which the iteration we will introduce in the next section doesn't).
\end{enumerate}

\subsection{Approximating universality}

We finish up the theoretical part of the paper by showing that we can choose a practical model class for which $m_C^n$ approximates the universal distribution with $n \to \infty$.

We will use the class of LSTMs \cite{hochreiter1997long}. This is a powerful model class, for which many theoretical properties are known, and also a very practical model, for which a great deal of software and hardware support is available, which will serve us in the experimental section. 

Consider a Turing machine with a single work tape $T$ at some point in its computation. We can represent its configuration by four bit strings (the content of the input, conditional, work and output tapes) and five natural numbers (the positions of the heads on each tape plus the state). We can represent the operation of one step of the Turing machine by a function $f: (\B)^4, \N^5,  \to (\B)^4, \N^5$.

We can set up an LSTM to compute $f_i$ as follows. We define an alphabet of three symbols: the two bits and a padding symbol. We represent these by one hot-coding, and concatenate the four tapes into one sequence of vectors, padding as necessary. We add an additional bit to indicate the position of the head, bringing the total dimensions to $16$. We encode the Turing machine's state in the initial value of the hidden vector.\footnotemark

\footnotetext{For TMs with very large numbers of states, this requires a minimum width or level of precision to represent. Since we need an LSTM only for the function $f$, we can set the width as necessary for the chosen precision and number of states in $T$.}

We add $k$ tokens with the padding symbol for all tapes (as "compute" tokens) and then read out the output. Since one step of the Turing machine requires a finite amount of computation, there is some sufficient value of $k$. 

On the output of the LSTM, we again produce 16-dimensional vectors, with a one-hot encoding per tape. We pass each output through a saturated linear activation

\[
\sigma(x) = \begin{cases}
0 & \text{if } x \leq 0\\
1 & \text{if } x \geq 1	\\
x & \text{otherwise.}\\
 \end{cases}
\]

We assume that the network which computes $f$ sets the pre-activations so that four bit strings result with the correct properties. 

Since LSTMs are Turing complete---they include Elman RNNs as a special case, which are Turing complete \cite{siegelmann1992computational,chung2021turing}---there exists for any $T$ an LSTM which computes this function.\footnotemark~Next, we would like to show that if we initialize an LSTM at random from a Gaussian over the parameters, that the probability of sampling a model which computes this function is not infinitesimal. 

\footnotetext{Note that we only need to compute one step of the TM.}

\begin{lemma}
Let $p(f)$ be the probability that an LSTM with $n$ parameters, initialized from a given non-degenerate Gaussian over its parameters computes the function $f$. If there exists one such initialization, then there exists some $\epsilon > 0$ such that $p(f) > \epsilon$
\end{lemma}
\begin{proof}
Let the parameter vector $\w$ represent an LSTM which computes $f$. It suffices to show that there exists some ball $B = B_r(\w)$ with radius $r$ such that all LSTMs in $B$ compute $f_i$ as well.

We first focus on the production of the preactivations of the output. This is done by some matrix and vector  $\y = \W_o\z + \bb$ with their elements taken from $\w$. By assumption, the outputs after activation are in the saturated part of the linear sigmoid. We can put them \emph{strictly} inside the saturated part with the operation $y' = 3y - 1$, or $\y = 3\W_o\z - {\mathbf 1} + \bb$. So, by multiplying all elements of $\W_o$ by 3 and subtracting $1$ from all elements of $\bb$, we have another LSTM $\w'$ whose outputs are strictly inside the saturated part of the output, at least 1 unit away from the edges. 

This means that all pre-activations in the ball $B_1(\y)$ also result in the correct output. We can now work backward through the network to establish which perturbations---expressed as the Euclidean distance between the original and perturbed output---we can allow for both the parameters of the operation and the inputs in order that the perturbation on $\y$ stays within the ball. For our purposes, it suffices to show that these perturbations are always nonzero. 

The \textbf{sigmoid and tanh activations} are $1/4$ and $1$ Lipschitz respectively. For a maximum output perturbation of $p$ we should allow a maximum input perturbation of $4p$ and $p$ respectively. 

The \textbf{addition of the bias vector}: for a maximum output perturbation of $p$ we must ensure that the input perturbation and the bias parameter perturbation do not add up to a vector longer than $p$, which we can ensure by limiting each to length $p/2$.

For any \textbf{matrix multiplication} $\W\x$, apply a perturbation of $\bE$ to the matrix and $\e$ to the input. The perturbation on the output is 

\begin{align*}
\| (\W+\bE)(\x + \e) - \W \x \| &= \| \W\e + \bE \x + \bE\e\|  \\
&\leq \| \W \| \|\e\| + \|\bE\| \|\x\| + \|\bE\| \|\e\|
\end{align*}
where the last line uses the triangle inequality and the submultiplicativity of the matrix norm. We use the Frobenius norm for matrices and the Euclidean norm for vectors. 

Note that all inputs to all matrix multiplications in the LSTM are bounded (our input vectors are binary vectors, and the inputs to any matrix multiplication are passed through a tanh or sigmoid activation). This means that there is an $m$ such that $\| \x\| < m $. We can now restrict the allowed output perturbation\footnotemark to $p < (m + \|\W\| + 1^{-1}$ and $p < 1$ and choose parameter perturbation $0 < \|\bE\| \leq p$ and input perturbation $0 < \|\e\| \leq p$.

\footnotetext{This requires us to lower the allowed output perturbation, but only to another non-zero value.}

 which gives us:
\[
\| (\W+\bE)(\x + \e) - \W \x \| < \|\W\| p + p m + p^2 < p(m + \|\W\| + 1) = p \p 
\]

We have shown that we can apply non-zero perturbations to every element of $\w'$ while keeping the behavior of the LSTM the same. If we take the smallest allowed perturbation $p$ encountered and apply it to all dimensions, we see that all LSTMs in the ball $B_p(\w')$ compute $f$. 

This ball receives a non-infinitesimal probability under any non-degenerate Gaussian, which proves the lemma. 
\end{proof}

Now let $C_\text{LSTM}$ be the class of Turing machines computing a forward pass of an LSTM. Specifically, for $c \in C_\text{LSTM}$, $c(x, y)$ computes the forward pass of an LSTM on the binary input matrix in column-major ordering $y$ on the conditional tape, ignoring any random bits on the input tape $x$, producing the output as described above in column-major ordering.\footnotemark

We assume that any non-bit outputs of $c \in C_\text{LSTM}$ are rounded to the nearest bit so that the input and output are always in $\B$.

\footnotetext{There are two levels of simulation here: each element of $C_\text{LSTM}$ is a Turing machine simulating an LSTM. Some small proportion of these compute $f$, which simulates a single forward pass of some chosen Turing machine.}

We assume a normal distribution as initialization, and sufficient floating point precision to ensure that $f$ is included in $C_\text{LSTM}$. With this construction, we can show that iterating the right model class converges to a universal class. 

Let $r$ be a readout function $x = r(\cdot, y)$ which takes a matrix bitstring representation $y$ of the four TM tapes (as defined above) and extracts the bits $x$ on the output tape. Similarly, let $s$ be a setup function $x = s(\cdot, y)$ which takes a string $y$ and turns it into a four-tape representation with $x$ on the input tape and all other tapes empty. 

\begin{theorem}\label{theorem:universality}
Let $m^n_\text{UTM}(x)$ be the distribution defined by running a universal Turing machine (UTM) on a random input for $n$ steps, and observing the output $x$. If $m_C$ dominates $C_\text{LSTM}$ and $r, s\in C$, then $m^{n+2}_C$ dominates $m^n_\text{UTM}$. 
\end{theorem}
\begin{proof}
\begin{align*}
m^{n+2}_C(x) &= \sum_{c, u_0} m_C^{n+1}(u_0) p(c) p_c(x \mid u_0)\\
&\geq \sum_{u_0} m_C^{n+1}(u_0) p(r) p_r(x \mid u_0) \\
&= \sum_{u_0} p(r) p_r(x \mid u_0) \sum_{c, u_1} m_C^{n}(u_1) p(c) p_c(u_0 \mid u_1) \\
&\geq \sum_{u_0, u_1} p(r) p(f) p_r(x \mid u_0) p_f(u_0 \mid u_1) m_C^{n-1}(u_1) \\
&\geq \sum_{u_0 \ldots u_n, u_{n+1}} p(r) p(f)^n p(s) p_r(x \mid u_0) \left(\prod_{t=1}^n p_f(u_{t-1} \mid u_t)\right) p_s(u_n, u_{n+1}) \;\; p_u(u_{n+1}) \p \\
\end{align*}

Note that $p_f$, $p_r$ and $p_s$ are deterministic, giving probability $1$ to one specific output for a given conditional. As such, from $u_{n+1} = u$, only one sequence has non-zero probability, so that we can rewrite

\begin{align*}
m^{n+2}_C(x) &\geq p(r)p(i)^np(s)\; \sum_{u} p_u(u) p_{UTM}^n(x \mid u) \\
&= p(r)p(f)^np(s)\; m^n_\text{UTM}(x) \\
\end{align*}

where $p_{UTM}^n(x \mid u)$ is the probability of observing $x$ on the output tape of the UTM after placing $u$ on the input tape and letting it run for $n$ steps.
\end{proof}

Note that as $n$ increases, the constant of domination $p(r)p(f)^np(s)$ decreases exponentially. This means that as $n \to \infty$, the constant becomes infinitesimal. Thus, while we dominate we resource-bounded $m^n_UTM(x)$ for any bound $n$, we cannot say that in the limit $m^\infty_C(x)$ dominates the unbounded universal distribution $m(x)$. 

For such a result, our LSTM would need to be able to simulate an arbitrary number of steps of the UTM rather than one step. This is possible \cite{chung2021turing} but requires arbitrary memory and time. The above result shows that even with a very modest model class---whose individual members are time- and memory-bounded and have finite precision---we can still approximate the universal distribution, in this specific sense.

\section{In practice}

To put these ideas in practice, we sample a random LSTM, $c$, mostly by standard initialization, and sample from it autoregressively while conditioning on a separate sequence of random data $z$. By the notation introduced above, the resulting sequence $x$ is sampled from $m^1_\text{LSTM}$ if $z$ is drawn from a uniform distribution, and $m^n_\text{LSTM}$ if $z$ is drawn from $m^{n-1}_\text{LSTM}$. Note that a new LSTM should be sampled for each $x$.

We then train a standard autoregressive transformer model \cite{radford2018improving} on sequences sampled in this way. 

Instead of training on sequences of bits, we use a token space of 256 characters. This makes it easier to test on downstream tasks by representing these in any type of (extended) ASCII encoding. 

For ease of implementation we limit ourselves to sequences of a fixed length (set to 512 in all experiments). This makes memory use predictable. 

\subsection{Source model}

For our model class we use single layer LSTMs with a hidden size $d_h$. We use the standard initialization (from $U(-h^{-\frac{1}{2}}, h^{-\frac{1}{2}})$) and multiply the weights (but not the biases) after initialization by a random multiplier drawn from $U(0, 1.1)$ for each model. We project the LSTM outputs to 256 dimensions with a linear layer and apply a softmax function. We add an embedding layer also with dimension $d_h/2$ for the inputs, which is initialized from a normal distribution.

We provide the LSTM with a seed $s$ (of 8 tokens in all our experiments) and a conditional sequence of $n$ tokens. We sample autoregressively, concatenating the current input sequence and the corresponding part of the conditional sequence after embedding.

For efficiency and ease of implementation, we initialize a single model $c$, and sample a batch of instances from it. This means that the batch is \emph{not} independently sampled from $m_C$. Training directly on such batches would create a bias at each gradient update for the patterns that are specific to $i$. 

To approximate independent sampling, we keep a buffer of $b$ samples. This is initially filled with random noise: half with \emph{constant} sequences (which repeat a single randomly chosen character) and half with \emph{noise} sequences (which sample each token independently, uniformly random).

Each iteration, we take $n$ instances $z$ and seeds $z$ from the buffer and replace them by samples from the current LSTM, conditioned on $z$ and seeded with $s$. For the seeds we slice 8-token subsequences from a random place in the sequence. We then sample another $n$ random instances from the buffer for our target model to train on. 

While this does not fully ensure that the samples are independent, it does ensure that with high probability, each instance in a training batch was generated from a different model $c$.

Moreover, because we are sampling structured noise, and feeding it back into the data, we are approximating a sample from a mixture $m^\text{mix}_C$ over the infinite sequence of models $m^1_C, m^2_C\ldots$, at the cost of a single sample from $m^1_C$.\footnotemark~

We call the distribution we are thus sampling from $m^\text{mix}_C(x)$. $m_C^\text{mix}$ dominates any $m_C^n$ albeit with a constant of domination that vanishes exponentially with $n$. The ablation in section \ref{section:ablation} shows that this resampling from the buffer has a beneficial effect.

The complete algorithm is detailed in Algorithm~\ref{algorithm:buffer}.

\begin{algorithm}

\begin{alltt}\ttfamily

\textbf{function} rsequences(len, con, ran):
\tab \textit{# Sample a set of random sequences}
\tab \textit{# len: sequence length, con: nr of constant seqs, }
\tab \textit{# ran: nr of random seqs}

\tab result \is [] 
\tab result.append(\textit{`con` constant repetitions of a random token})
\tab result.append(\textit{`ran` sequences containing random tokens})
\tab \textbf{return} shuffle(result)

buffer = rsquences(n, buffer_size/2, buffr_size/2)

\textbf{function} sample_batch(bs, reset):
\tab \textit{# Sample a training batch }
\tab \textit{# bs: batch size, nr of instances to reset}

\tab c \is LSTM(\ldots) # initialize model

\tab \textit{# sample instances to enrich}
\tab idx \is  sample(bs, buffer_size) # sample bs indices
\tab seqs  \is buffer[idx]
\tab seeds \is \textit{`bs` subsequences from anywhere in the buffer}

\tab \textit{# sample sequences from LSTM}
\tab seqs \is autosample(c, seeds=seeds, conditional=seqs) 

\tab reset_idx \is sample(reset, bs)
\tab seqs[reset_idx] \is rsequences(reset/2, reset/2)

\tab buffer[idx] \is seqs

\tab \textit{# sample batch}
\tab idx \is sample(bs, buffer_size)
\tab \textbf{return} buffer[idx]

\end{alltt}

\caption{Pseudocode for the buffering algorithm}
\label{algorithm:buffer}	
\end{algorithm}

On these samples, we train a standard autoregressive transformer (architecture and training details in the appendix) by shifting the sequences one character to the left to form the target, so that the model is trained to predict the next character at each position, with negative log loss as the training objective. As shown in \cite{muller2021transformers}, this objective is minimized when the transformer's distribution over sequences coincides with our prior, $m^\text{mix}_C(x)$.

\section{Experiments}

\subsection{Downstream data}

We evaluate on various synthetic sources of data. Each consists of a generator that produces strings of ASCII characters  (mapped to integers in $[0, 255]$), possibly of variable length.  We sample sequences and optionally concatenate them with a single delimiter character ``\texttt{|}'' until the total exceeds 100\h000 characters.

Note that the model is always over all 256 tokens, so that its first task is to home in on the characters used by the specific generator.
 
To provide a reasonable bar to test basic universal pre-training, we design the datasets around an intuition of what a Markov model---either trained on a large dataset from this source, or in-context from a short, specific sequence---should be able to achieve. Our first aim is to show that our model can, at least in some instances, do better than a Markov model. 

We treat some of these as validation sets (v), on which we tune our hyperparameters, and some as test (t), only evaluated in the final experiments. 
 
\begin{description}
\item[\texttt{champ} (v)] Inspired by the Champernowne constant,\footnotemark~this source generates successive integers, and outputs their concatenated digits as tokens. We sample a starting integer at random from $[0, 16777216-n]$ and then concatenate it together with the subsequent 255 integers in a string, the characters of which are tokens. The idea behind this dataset is that there is a simple computational structure that allows good predictions, but a simple, low-order Markov model trained on the training data will not be able to exploit this structure (and an in-context Markov model only partially).
\item[\texttt{dyck} (v)] Generates \emph{Dyck words}: that is, sequences of opening and closing brackets, such that each opening bracket is closed with no closing brackets left over. The result is that a simple predictor can get good performance by always predicting ``('' or ``)'' with equal probability and ``\texttt{|}'' with small probability (as a $0$-order Markov model would do), while a perfect predictor would only assign non-zero probability to ``\texttt{|}'' only when the preceding string forms a Dyck word. 
\item[\texttt{ndfa} (v), \texttt{aut} (t)] A string generated from a simple, non-deterministic automaton, described in Figure~\ref{figure:sources}.
\item[\texttt{toy} (v), \texttt{toy2} (t)] A string generated from a simple toy grammar, described in Figure~\ref{figure:sources}.
\item[\texttt{bits}] Each string consists of two random strings of 5 bits, followed by the result of applying the boolean operators xor, $\wedge$, $\vee$ and $=$ respectively for a sequence length of 30 bits. The first 10 of these are fully random, while the last 20 are fully determined by the first 10. Including the delimiter, a perfect predictor, given sufficient context would get a performance of $10/31 \approx 0.32$ bits.
\item[\texttt{bitsflip}] Repeatedly: a random string of 6 bits followed by the same string reversed, with no delimiter. Can be optimally predicted in 0.5 bits/char.
\end{description}

\footnotetext{$0.\rcl{1}2\rcl{3}4\rcl{5}6\rcl{7}8\rcl{9}10\rcl{11}12\rcl{13}\ldots$ One important property is that no $n$-gram occurs more frequently than any other in the decimal expansion. That means a Markov model would not be able to predict the next character better than chance. However, there is a simple program that can perfectly predict the next character, given the natural ordering of the digits.}

\begin{figure}[t!]

  \centerline{
  \includegraphics[width=1\textwidth]{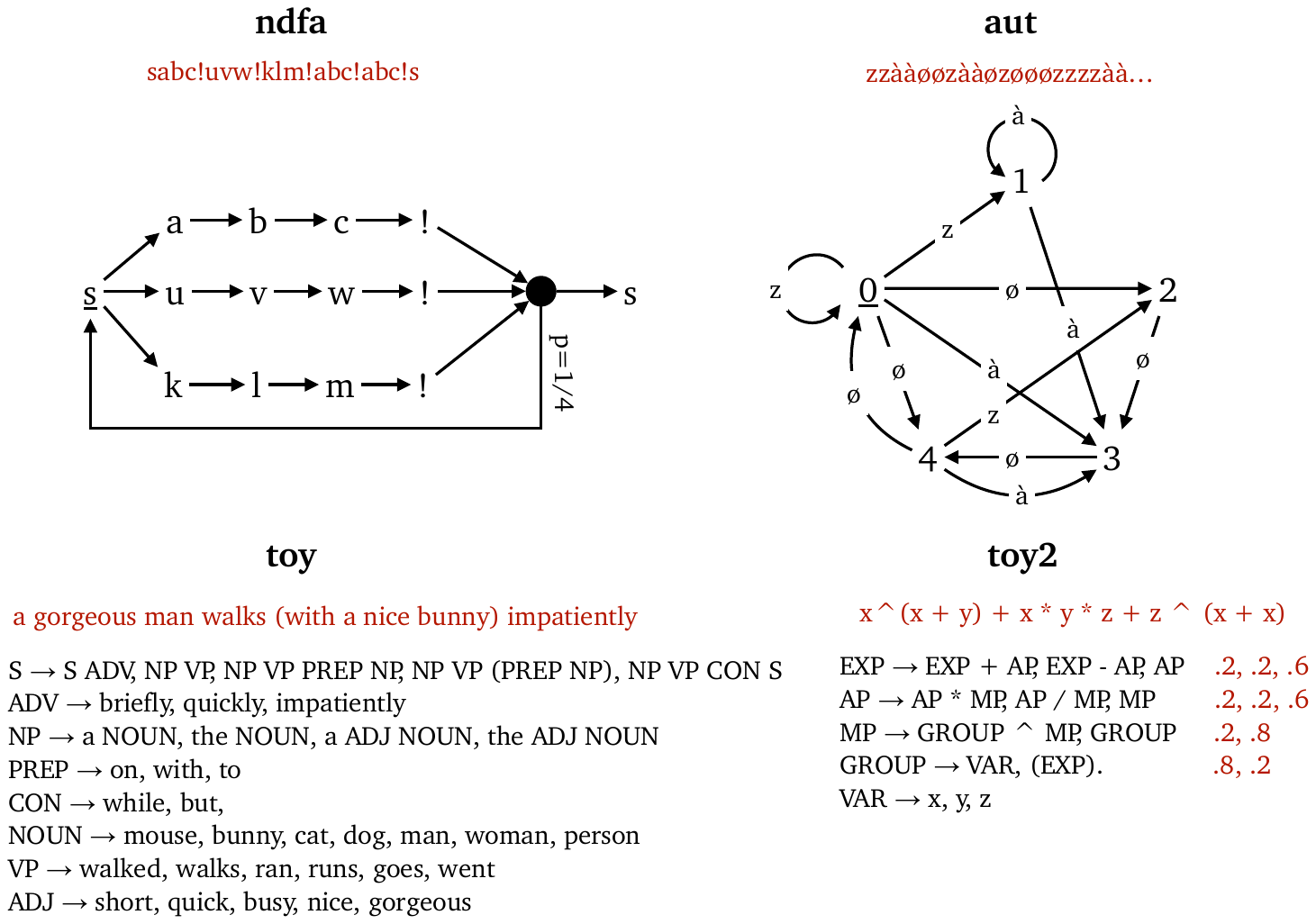}
  }
  \caption{The four toy data sources used. Example outputs are shown below the name in red. (ndfa) A simple hand-designed non-deterministic automaton (drawn in non-standard manner for the sake of simplicity). (aut) A randomly generated non-deterministic automaton. Transition probability are uniform. (toy) A hand-designed toy grammar. (toy2) A toy grammar adapted from \cite{heckendorn2021practical}. Probabilities in red are over the symbol replacement options (uniform if not specified). }
  \label{figure:sources}
\end{figure}

In addition, we include real-world datasets. Each comes with a validation/test split.

\begin{description}
\item[\texttt{wp} (v/t)] The validation and test set of the original Hutter prize Wikipedia data \cite{hutter2006hutter} (also known as \texttt{enwik8}). The data is loaded as a sequence of bytes, so that some special characters are represented as two tokens. 
\item[\texttt{german} (v/t)] The text of the public domain book \emph{Die Frauenfrage} \cite{braun1901frauenfrage}, chosen for containing a relatively large amount of structured text (in the form of ASCII formatted tables) in addition to natural language. We used the text-file version from Project Gutenberg, removed the Gutenberg pre- and post-matter and loaded the text as a sequence of bytes.
\item[\texttt{code} (v/t)] The minified javascript code of the D3.js library (version 7.9.0) \cite{d3js}. Chosen for containing a reasonable amount of code, but (as a javascript project) being easily available in minified form so that a larger amount of structure will be contained in the relatively modest context length of our models. Contains no natural language comments.
\item[\texttt{linux} (v/t)] The first 25 megabytes of the linux kernel codebase (commit \texttt{dd83757f6e}) \cite{linux}. Chosen to allow us to test finetuning on a coding dataset of similar size to the wikipedia data. Contains substantial amounts of natural language comments in addition to the \texttt{C++} code.
\end{description}

\noindent We estimate the binary negative log-likelihood of a given model on this data. The characters are mapped to the tokens of the model by their ASCII code.\footnotemark~We sample 10\h000 slices of length $n$ (the context length) from the data and average the negative binary log-probability that the model assigns to the character in the last place of the sequence. Note that we treat the whole data as a single string, so that delimiters are included (and may end up anywhere in the sampled slice).

\footnotetext{This mapping is arbitrary, since the model has not been trained on this data.}

\subsection{Zero-shot performance and scaling}

\label{section:scaling}

\begin{figure}[t!]

  \centerline{\hspace{-1.2em}
    \includegraphics[width=1.6\textwidth]{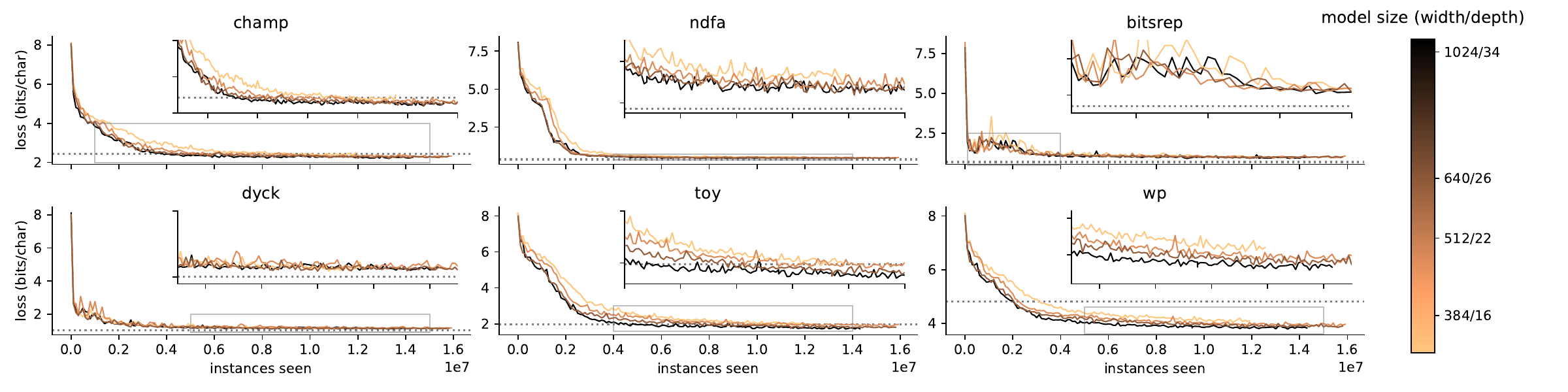}
  }
  \caption{The results of the main experiment (Section~\ref{section:scaling}) on the validation datasets.}
  \label{figure:scaling-val}
\end{figure}

We train models of increasing size to study scaling behavior. To scale up the model we first set the width $w$ (also known as $d_\text{model}$, the dimension of a token representation going into each transformer block). We then choose the depth--the number of transformer blocks---$L$ as suggested in \cite{levine2020depth} by the formula $L = \lfloor\frac{\ln w - 5.039}{0.555}\rfloor$. We use maximal update parametrization (MUP) \cite{yang2021tuning} so that we can tune the hyperparameters at $w=384$ and adapt the larger models without tuning. We scale the width of the LSTM model proportional to the width of the target model.

We keep the context fixed at $512$ for all models\footnotemark and use $w/128$ attention heads. Further architectural and training details are given in the appendix. 

\footnotetext{The limited context length may be responsible for a slight plateauing effect in the scaling, since larger contexts may be required to scale up prediction with the model. We leave this to future work.}

We pre-train on approximately 20M instances sampled by the method detailed above. At every 100K instances, we evaluate the model (without any training) on all downstream datasets.

Figures \ref{figure:scaling} and \ref{figure:scaling-val} show the results on the test and validation sets respectively. The dotted line shows the performance of a $k$-th order Markov model, for the optimal $k$ and Laplace smoothing parameter $\lambda$ chosen for the data by exhaustive tuning (over orders 0-5 and smoothing values $1$, $0.1$, $0.01$, $0.0001$, $10^{-6}$). The Markov model is trained on the same context as the model is given. The hyperparameters are chosen once per dataset.

The results show that the model improves substantially over a random baseline (which would score 8 bits per token on all data). It does not outperform the Markov model on most of the synthetic datasets, but it does do substantially better on the real-world datasets. This shows promise for the ideal of universal pre-training in real-world models, while also showing that there are still many challenges left to overcome. 

Among the synthetic datasets, the model does dip slightly below the Markov boundary for the \texttt{champ} and \texttt{toy} datasets. 

In addition, we observe improvement with scale in the real-world datasets. This suggests a possibility of a scaling law for universal pre-training, which we discuss further in the last section.

\subsection{Finetuning}

\label{section:finetuning}

\begin{figure}[t!]

  \centerline{\hspace{-1.2em}
    \includegraphics[width=1.6\textwidth]{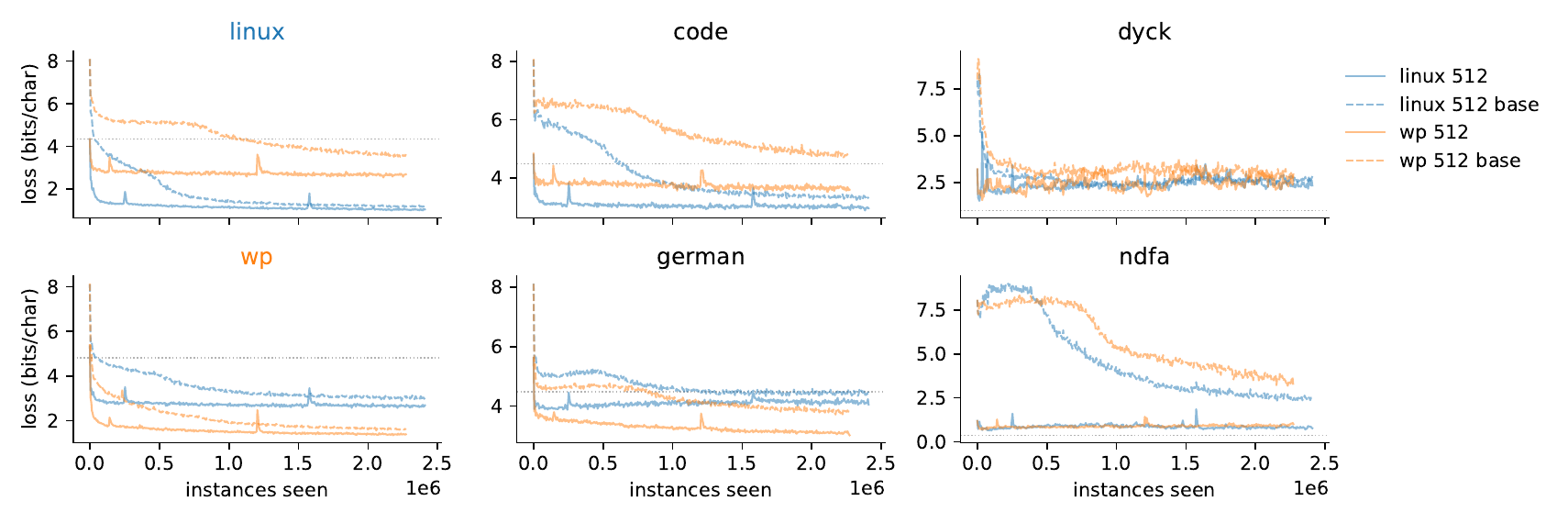}
  }
  \caption{The results of the finetuning experiment. Colors indicate the dataset that was finetuned on (wikipedia or linux). Dashed lines show baselines (the same model, but with the parameters re-initialized). Note that as finetuning progresses, performance on other data is often largely retained. 
  }
  \label{figure:finetuning}
\end{figure}

Next, we test the effect of taking one of our models and further finetuning it on non-synthetic data (specifically the \texttt{wp} and \texttt{linux} training datasets). We find the largest model on which we can finetune in a straightforward manner has $w=512$.\footnotemark~We take the final model at the end of the training run from Section~\ref{section:scaling} and tune it further on the specific training dataset. For a baseline, we re-initialize the model parameters using the same scheme used for pretraining. We choose the highest learning rate ($10^{-4}$) for which the baseline does not destabilize and use it for all models. 

The results are shown in Figure~\ref{figure:finetuning}. We first observe that universal pre-training confers an advantage in the speed of convergence. Both models converge to the a similar optimum, but the pretrained model gets there about 1M instances faster than the baseline.

\footnotetext{The $w=1024$ model trains well initially, but has a certain probability of destabilizing. Most likely this can be solved with better hyperparameter tuning or adapters, but resource limitations require us to leave this question to future work.}

Of course, compared to the baseline, the pre-trained model has the benefit of training on aroudn 20M synthetic instances which we do not include in the comparison. This is justified in a setting where we can amortize the pre-training over many downstream tasks, making its cost arbitrarily small. 

We find that with a higher learning rate ($3\cdot 10^{-4}$), the baseline and pre-trained model converge to exactly the same value, but both models occasionally destabilize. This suggests that while in the shown experiment, the pre-trained model and the baseline never quite touch, this gap may disappear with more careful finetuning. 

Beyond the faster convergence, we see that the pre-trained model largely retains its performance on datasets other than the finetuning target. This suggests a \emph{generalization} benefit: a pre-trained model, finetuned on in-domain data, may generalize better to out-of-domain data, at least to the extent that the pretraining distribution covers the out-of-domain datasets. 

Exactly how much generalization benefit one gets in practice, and how well this works for practical generalization problems, we leave to future work. 

Finally, we note here that the Wikipedia and Linux datasets themselves appear to confer some pre-training benefit on the other datasets. This is likely down to the wide variety of structures present in the data beyond simple natural language, making these sources something of an approximation to the universal distribution themselves. 

\subsection{Ablation}
\label{section:ablation}

Finally, we ablate a selection of design choices made in our final experiment. The results are shown in Figure~\ref{figure:ablation}.

\begin{figure}[t!]

  \centerline{\hspace{-1.2em}
    \includegraphics[width=1.6\textwidth]{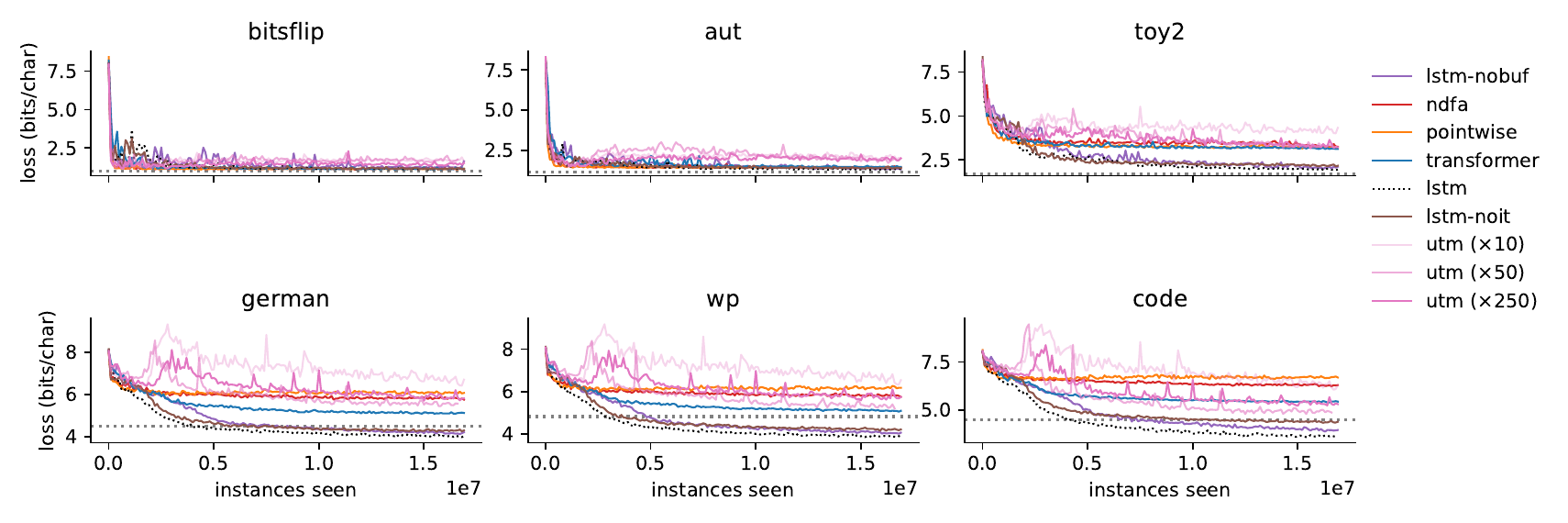}
  }
  \caption{The results of the ablation experiment (Section~\ref{section:ablation}). Starting with the base experiment for $w=384$ (in black, dotted), we try three alternative source of structured random data: a \texttt{transformer} model, a random deterministic automaton (\texttt{ndfa}), and a random distribution on the character space, iid over the time dimension (\texttt{pointwise}). In addition, we try our LTSM source without buffer (\texttt{lstm-nobuf}) and without iterating the model to increase the depth of the structure present (\texttt{lstm-noit}).. 
  }
  \label{figure:ablation}
\end{figure}

%\paragraph{Source depth} In our scaling experiments, we scale the source model up at the same rate as the target model. While this could in theory lead to more computationally rich data, this is not guaranteed, and it may be that a model of 40-layers deep will learn that same patterns from a 40 layer source model as it would from a 3-layer source model. 
%
%To investigate, we retrain the model of width 1536 and depth 41 and re-train it with 5 source models of different width. We adjusting all other parameters in the same way as in our scaling experiment. The results are shown in Figure~\ref{figure:ablation-1}.

% REDO in version 2

First, we train on three further sources of random data: 
\begin{description}
	\item[pointwise] For each instance, we sample a categorical distribution over the tokens from a Dirichlet prior with $\alpha=\frac{1}{2}$. We then sample each token independently from this distribution. 
	\item[ndfa] A randomly generated non-deterministic automaton. 
	\item[transformer] We sample from a transformer model in the same way we sample from the LSTMs. While this creates a pleasing architectural symmetry between the source and target model, the transformer relies on learned embeddings for its sequential bias, which makes it very unlikely that this bias emerges by random sampling.
\end{description}
	
\noindent In addition, we ablate the buffering mechanism (\texttt{lstm-nobuf}) by making the buffer the same size as the batch size, and the iteration of generated data back into the lstm (\texttt{lstm-noit}), by resetting the sampled batch to half constant/half uniform noise before feeding it to the LSTM.

Finally, we test on the explicit resource-bounded UTM used in \cite{grau2024learning} as a source. We scale the resource bounds (on memory, program length, and time) up by factors of 10, 50, 250. At 50, the UTM takes slightly less time per sample than the LSTM and at 250 slightly more.

It may be that a large reason for the difference in preformance between the LSTM and the UTM is down to the fact that the LSTM can be accelerated by GPU, and so gets more computation for the same amount of time. It may also be that its natural inductive bias (the prior it assigns to functions) is better suited for natural language data like \texttt{wp}.

It may be that mixing sources together---for example feeding the output of a UTM to an LSTM---provides the best of both worlds. The mechanism of Section~\ref{section:theory} provides a simple opportunity for this. We leave such investigations to future work. 

\section{Discussion}

There are three main ingredients in the modern machine learning process that limit how well we do: the fundamental abilities of the algorithm, the available amount of data and the available amount of computational resources. This paper shows that, at least to some extent, it is possible to trade off computational resources for data. We can generate structured random data to pre-train our models, making them require less real-world training data by comparison.

Moreover, the pre-training is in some sense universal. That is, the same pre-training can benefit a large amount of downstream tasks. This allows us to amortize the pre-training over many different tasks, reducing the cost of the tradeoff.

\subsection{Limitations and Future Work}

\paragraph{Universality of assumptions} Calling the specific approach used in our experiments \emph{universal} is a slight overstatement. We do not sample from a universal distribution (which is not tractable). A class-universal distribution with a sufficiently broad class (like the polynomial-time class $P$) would still justify the name universal pre-training (since it's very likely that all data sources are in that class), but we fall short of that as well.

First, in the interest of efficient generation, our class $C$ is limited to finite strings. While the iteration trick allows arbitrary sampling depths (in theory), this does not necessarily mean that we are sampling from the distribution on arbitrarily deep LSTMs/TMs as we are in Theorem~\ref{theorem:universality}. First, because the probability of sampling at depth $n$ decays exponentially with $n$ in our algorithm and second because we limit the sequence-length to a fixed size.

In short, while the framework as we have described it, in its most general terms deserves the moniker \emph{universal}, our practical implementation of it is more of an approximation to this ideal. This is also reflected in our experimental results where simple tasks such as \texttt{bitsflip} and \texttt{toy2} are not solved by our model, while a truly universal model should have no problem with them.

We leave the notion of (more) true universal pre-training a a challenge to the community. This will likely require more investigation of models as well as more carefully designed benchmarks and experimental protocols to test that the model truly performs well on a variety of tasks on which it wasn't trained. We hope that the current work serves as a proof-of-concept that this is a worthwhile direction of investigation.

\paragraph{Other domains} Our experiments are limited to token sequences. That is, strings of tokens from a finite vocabulary. In principle, the same idea could be applied to computer vision, continuous time series and any other domains. Ideally, those domains would be unified in a single data representation so that one universal model may be applied to all.

\subsection{Social impact}

At the time of publication, large model development is experiencing explosive growth. Some commentators have already suggested that this is a bubble, driven by overoptimistic predictions of future payoffs. While energy use of individual models is still relatively modest, compared to other areas of scientific investigation, the race for larger and better models has already led to technology companies abandoning their promises of carbon neutrality \cite{milmo2024google} and investigating the use of nuclear power purely for the purpose of meeting the energy demands of AI \cite{luscombe2024three}. 

The limited availability of high-quality data functions as a limiter on this exponential growth. Training models beyond a certain size is no longer feasible, not because companies are not willing to invest the money and energy, but because the data is either not available, or it is too expensive to acquire at scale.

We have shown that, at least in principle, compute may be traded off against data. We stress that it remains to be seen how this effect plays out at scale, and it may never become economical to do so. However, if it does turn out to be a viable tradeoff, an important limiter on the explosive investment will disappear. 

We need only look at developments like crypto-currency to see how far people  will collectively go to invest in a suggested future reward, with amounts of energy at the scale of countries going into a technology that has far from proven its value let alone its return on investment \cite{kohli2023analysis}. In short, there is a risk of an even greater run on sources of energy for training AI models than we are already seeing, at a time when climate change dictates that we should rather be strictly scrutinizing our energy consumption.

On the other hand, universal pre-training may also provide a middle ground. Since a model trained in this way is guaranteed not to contain any sensitive data, or any data-driven social biases, and can be used in a variety of tasks, the training of such models may be centralized, with the result easily published under a free license. While the cost of training will still be substantial, it can then be amortized over all users, rather than just the use regulated and sold by a single company. So long as we can prevent a multitude of companies duplicating the same effort in a race to the biggest proprietary model, the net impact may be a positive, reducing the amount of energy the average user spends on training AI models.

\paragraph{Acknowledgments} Experiments for this research were run on the DAS-6 \cite{das6} and Snellius supercomputers.

\bibliographystyle{splncs04}
\bibliography{up.bib}

\appendix
\section{Appendix}

\subsection{Model and training details}

We trained the following models

\begin{tabular}{l l l l l}
width $w$ & depth & source width & target microbatch size\\
\hline
384 & 16 & 38 & 69 \\ 
512 & 22 & 51 & 48 \\ 
640 & 26 & 64 & 36 \\ 
1024 & 34 & 102 & 16 \\
\end{tabular}

\paragraph{Training} All experiments were performed on single A100 GPUs with 40 GB or VRAM, with the exception of the finetuning experiment for which H100 GPUs with 80 GB of VRAM were used. Training was completed in a single run of 5 days for the 384 and 512 width models, two such runs for the 640 model and three such runs for the 1024 model.

We tuned the model at $w=640$ and found a base learning rate of $3\cdot 10^{-4}$, which we transferred to the other models by MUP. We warm up the learning linearly over the first 1.6M instances and then cool down the learning rate geometrically so that it halves every 10M instances. We use the AdamW optimizer \cite{loshchilov2017fixing} with a weight decay of $0.01$. 

We clip gradients by scaling every gradient vector larger than 1 to norm 1. 

The maximum macrobatch size is 500 in all experiments. We start at a batch size equal to the microbatch size, and warm up linearly to the maximum macrobatch size between 500K and 1M instances. 

In the finetuning experiment we reduce the learning rate warmup to 10K instances and warm up the macrobatch between 50K and 150K instances, for both the baseline and the pre-trained model. The base learning rate (under MUP) is set to $10^{-4}$.

We set the buffer to 20 times the source batch size---the number of instances from the buffer passed through the a random LSTM for each iteration---which is 138 for all models. In each source batch, we reset 20 instances to constant random noise and 20 instances to random noise after feeding the instance through the LSTM. 

We train the model in mixed precision. Unless otherwise noted, default Pytorch initializations are used. 

\paragraph{Source architecture details} We use the default LSTM implementation from Pytorch with a single layer. The inputs are generated by an embedding layer (initialized from a normal distribution, and the outputs by a single linear layer with default initialization ($U(-d^{-\frac{1}{2}}, d^{-\frac{1}{2}})$ with $d$ the output size of the LSTM. 

We sample from the LSTM with temperature $10^{-4}$.

\paragraph{Target architecture details} Our transformer block consists of a standard self-attention followed by a layer normalization, followed by a feedforward network followed by a layer normalization, with residual connections around the self-attention and feed-forwards. We apply 0.1 dropout after each layer normalization. The feedforward has a single hidden layer which is 4 times the size of the input/output. It uses a ReLU nonlinearity.

In the self attention, we follow \cite{vaswani2017attention}, but scale by the dimension $k$ rather than its square root.

We use position embeddings and a plain linear layer as our output head.

\paragraph{Progressive scaling} When scaling up, we find it helpful to disable some blocks at the start of learning. To enable this we create \emph{progressive blocks}. These consists of a normal block, with a residual connection which multiplies the block's output by a learnable scalar $a$, initialized to zero. At the start of training, $a$ is frozen. At some point $a$ is unfrozen, allowing the block to begin to contribute as the magnitude of $a$ slowly increases under gradient descent. The first 8 blocks in any model are normal, with progressive blocks added in groups of 8 until the required depth is reached. During training, every 10K instances the next earliest group of progressive blocks that is currently disabled is enabled until the whole network is enabled.  

\end{document}